\documentclass[11pt]{article}

\usepackage[margin=1.0in]{geometry}

\usepackage[authoryear,round]{natbib}

\usepackage[utf8]{inputenc} 
\usepackage[T1]{fontenc}    
\usepackage{url}            
\usepackage{booktabs}       
\usepackage{amsfonts}       
\usepackage{nicefrac}       
\usepackage[protrusion=true,expansion=false]{microtype}

\usepackage{algorithm}
\usepackage{algpseudocode}
\algtext*{EndFor}
\algtext*{EndIf}
\algrenewcommand\algorithmicrequire{\textbf{Input:}}
\algrenewcommand\algorithmicensure{\textbf{Output:}}
\algrenewcommand\algorithmicdo{} 
\algrenewcommand\alglinenumber[1]{}   
\usepackage{amsmath}    
\usepackage{amsfonts}
\usepackage{amssymb}    
\usepackage{amsthm}     
\usepackage{bm}         
\usepackage{enumitem}
\usepackage{wrapfig}
\usepackage{graphicx}

\usepackage{afterpage}
\usepackage{authblk}  

\newcommand{\prox}{\operatorname{prox}}

\usepackage[dvipsnames]{xcolor}
\usepackage[
    colorlinks=true,
    linkcolor=RoyalPurple,   
    citecolor=Blue,     
    urlcolor=Violet     
]{hyperref}
\usepackage[most]{tcolorbox}
\usepackage{varwidth}
\usepackage[capitalize]{cleveref}
\newtcbox{\mybox}{colback=blue!5,
	colframe=blue!30!black, center, enhanced, varwidth upper}
\usepackage{stmaryrd}
\definecolor{myblue}{RGB}{0, 50, 200} 
\definecolor{myorange}{RGB}{150, 50, 50} 

\usepackage{xspace}

\crefformat{equation}{(#2#1#3)}
\makeatletter
\newcommand{\oset}[3][0ex]{%
  \mathrel{\mathop{#3}\limits^{
    \vbox to#1{\kern-2\ex@
    \hbox{$\scriptstyle#2$}\vss}}}}
\makeatother

\newcommand{\x}{\mathbf{x}}
\newcommand{\y}{\mathbf{y}}
\newcommand{\z}{\mathbf{z}}
\renewcommand{\b}{\mathbf{b}}   
\newcommand{\ones}{\mathbf{1}}
\newcommand{\T}{\mathcal{T}}

\newtheorem{lemma}{Lemma}

\theoremstyle{remark}

\crefname{assumption}{Assumption}{Assumptions}

\title{Learning Affine-Equivariant Proximal Operators}

\author[1]{Oriel Savir}
\author[1]{Zhenghan Fang}
\author[1]{Jeremias Sulam}
\affil[1]{Mathematical Institute for Data Science, Johns Hopkins University, Baltimore, MD 21218, USA}
\date{}

\begin{document}

\maketitle

\begin{abstract}
  Proximal operators are fundamental across many applications in signal processing and machine learning, including solving ill-posed inverse problems. Recent work has introduced Learned Proximal Networks (LPNs), providing parametric functions that compute exact proximals for data-driven and potentially non-convex regularizers. However, in many settings it is important to include additional structure to these regularizers--and their corresponding proximals--such as shift and scale equivariance. In this work, we show how to obtain learned functions parametrized by neural networks that provably compute exact proximal operators while being equivariant to shifts and scaling, which we dub Affine-Equivariant Learned Proximal Networks (AE-LPNs). We demonstrate our results on synthetic, constructive examples, and then on real data via denoising in out-of-distribution settings. Our equivariant learned proximals enhance robustness to noise distributions and affine shifts far beyond training distributions, improving the practical utility of learned proximal operators.

\end{abstract}

\section{Introduction}
\label{sec:intro}

Inverse problems are concerned with the task of inferring parameters from a given observation, commonly a recovery task after a degradation process, such as in image denoising and deblurring \citep{milanfar2024denoisingpowerfulbuildingblockimaging, evangelista2023ambiguitysolvingimaginginverse}, super-resolution \citep{candes2012mathematicaltheorysuperresolution}, and compressed sensing \citep{8015117}. These problems are often ill-posed, requiring regularization to provide a unique or stable solution \citep{benning2018modernregularizationmethodsinverse}. Classical approaches construct explicit regularizers to promote characteristics in the solution, such as smoothness \citep{doi:10.1137/1021044} and sparsity \citep{bruckstein2009sparse}. Purely learning-based methods, on the other hand, train models to minimize reconstruction losses between a ``corrupted'' input and ground-truth, such as in CNN-based image denoising \citep{NIPS2008_c16a5320}.

A third approach exploits the observation that many iterative solvers for inverse problems use the proximal operator as a sub-step \citep{10.1561/2400000003}. For a \emph{proximable}\footnote{Or weakly-convex function.} regularizer $R:\mathbb R^n \to \mathbb R$, its proximal operator at a point $\y\in\mathbb R^n$ is defined as
\begin{equation} \label{eq:prox}
    \prox_R(\y) = \arg\min_{\x} \frac12 \|\x-\y\|^2_2 + R(\x).
\end{equation}
Since the proximal operator can be loosely thought of as a denoising operation, frameworks utilizing denoising networks in place of the proximal have become very popular \citep{6737048,Meinhardt_2017,romano2017littleenginecouldregularization}.
These methods generally provide no guarantees of computing true proximal steps or recovering the correct regularizer for the data. Recent work \citep{fang2024whatspriorlearnedproximal} provides such guarantees via their \emph{Learned Proximal Networks (LPNs)}. These LPNs provide a means to train a denoiser that is a proximal operator by construction, resulting in a useful compromise between adaptability to data and analytical tractability. Furthermore, \citet{fang2024whatspriorlearnedproximal} provide a loss function, \emph{proximal matching loss} ($\mathcal{L}_{PM}$), that ensures the learned regularizer correctly matches the prior of training data (in an asymptotic way).

While useful, the regularizers that result from LPNs are general, and it is unclear how to enforce additional structural properties that can be useful for a given task, such as symmetry properties. When data exhibits symmetries, enforcing them in a regularizer can improve solutions to ill-posed problems. Moreover, learning functions that respect the symmetries present in data can reduce the statistical complexity of these learning tasks \citep{bietti2019group,anselmi2016unsupervised}. Equivariance, in particular, has been shown previously to greatly aid in generalization of denoising models \citep{mohan2020robustinterpretableblindimage, herbreteau2024normalizationequivariantneuralnetworksapplication}. A function is equivariant if transforming the input transforms the output accordingly. Formally, a function $f \colon \mathbb R^n \to \mathbb R^n $ is equivariant w.r.t to a group $G$ if, for every element of a group $g\in G$, $f(g(\x))= g(f(\x)), \forall \x \in \mathbb R^n$.
We focus on functions that are shift- and scale-equivariant, or \emph{scalar-affine equivariant}, i.e. when $f(a \x + \b) = a f(\x) + \b, \quad
\forall \x \in \mathbb{R}^n, \, a \in \mathbb{R}_+, \, \b \in \mathbb{R}^n$, where $\b$ is a constant vector, i.e. $\b = b \ones$. Such functions have been shown to be particularly useful in image denoising problems where noise-levels can change from training to deployment settings \citep{herbreteau2024normalizationequivariantneuralnetworksapplication}. Thus, the central question we address in our work is this: \emph{How can we design data-driven learned and scalar-affine equivariant proximal operators?}

Multiple approaches exist to realize affine-equivariant functions parametrized by neural networks. A straightforward option is the so-called \emph{transformation trick}: apply a suitable transformation $\T$ to the input, apply the network to the transformed data, and then invert the transformation via $\T^{-1}$. In image denoising, where the input and output share the same dimensions, a sensible choice of $\T$ is the mean-standard deviation normalization, $\T(\x) = \frac{\x - \mathrm{mean}(\x)}{\mathrm{stdev}(\x)}.$
Another approach is to design denoising networks that are intrinsically affine-equivariant. This has been recently done in \citet{herbreteau2024normalizationequivariantneuralnetworksapplication} by removing additive biases and introducing \emph{affine convolutions} combined with a sorting-based nonlinearity, \emph{SortPool}. Unfortunately, neither approach allows for a satisfying answer to our question. First, the normalization trick breaks the guarantee that the network represents a proximal operator\footnote{The normalization steps prevent the learned mapping from being the gradient of a convex function. See more in Section~\ref{sec:background}.}. Meanwhile, the architectural adaptations in \citet{herbreteau2024normalizationequivariantneuralnetworksapplication} are designed for general denoisers, whereas LPNs are defined implicitly as the gradient of a scalar-valued (and convex) network. Crucially, affine-equivariance does not transfer from a function to its gradient -- a function may be affine-equivariant while its gradient is not.

This motivates our proposed AE-LPN, providing LPNs (and thus, exact proximals) while preserving their equivariance structure to affine transformations. In particular, we identify conditions under which a convex function admits an affine-equivariant gradient mapping, and show how these conditions can be enforced in neural networks through input transformations and architectural design. By enforcing these conditions in an input convex neural network and taking its gradient, we obtain AE-LPNs that provably parameterize proximal operators while being affine-equivariant. We demonstrate the effectiveness of AE-LPNs in both low-dimensional and high-dimensional settings for learning characterizable, robust and generalizable priors.

\section{Background}
\label{sec:background}

Learned Proximal Networks (LPNs) \citep{fang2024whatspriorlearnedproximal}
parameterize proximal operators -- i.e., functions of the form of \eqref{eq:prox} -- by using their characterization as the gradient of a convex potential $\psi_\theta : \mathbb{R}^n \to \mathbb{R}$. Previous works have shown an equivalence between functions that are gradients of convex potentials and proximal operators \citep{gribonval2020characterizationproximityoperators}, showing that $f = \prox_R$ for some proximable $R$ if and only if there exist a convex potential $\psi$ so that\footnote{This result can be extended to non-smooth convex potentials $\psi$, see \citet{gribonval2020characterizationproximityoperators}.} $f = \nabla \psi$. LPNs \citep{fang2024whatspriorlearnedproximal} provide proximal operators directly as gradients of convex functions, $\psi_\theta$, parametrized by convex neural networks with parameters $\theta$. Convexity in $\psi_\theta$ is enforced by constraining network weights to be nonnegative and using convex, nondecreasing activation functions, following the design of Input Convex Neural Networks (ICNN) \citep{amos2017inputconvexneuralnetworks}. The resulting mapping $f_\theta(\x) = \nabla \psi_\theta(\x)$ is a conservative gradient field derived from a convex function, and therefore a proximal operator, i.e. $f_\theta(\x) = \prox_R(\x)$. Moreover, the implicit function $R$ can be evaluated at arbitrary points $\x\in\mathbb R^n$ by solving a (strongly) convex optimization problem \citep{fang2024whatspriorlearnedproximal}.

While LPNs provide proximal operators by construction, it is important to have these networks compute proximals for ``useful'' functions, or regularizers. In the context of inverse problems, such a choice corresponds to the log-density of the data distribution. This can be seen by understanding the definition of the proximal operators (see \eqref{eq:prox}) as a \textit{maximum-a-posteriori} (MAP) estimate for $\x$ given Gaussian-corrupted samples $\y = \x + \z$, where $\z \sim \mathcal N(0,\sigma^2I)$. The work in \citet{fang2024whatspriorlearnedproximal} provides a training scheme, dubbed \textit{proximal matching}, that precisely promotes the implicit regularizer to recover the log-density of the data in an unsupervised manner. Proximal matching relies on being able to sample ground-truth signals $\x\sim p_\x$, constructing Gaussian corrupted samples $\mathbf{y}$ as above, and minimizing a risk function that measures the ability of $f_\theta$ to denoise $\y$ and recover $\x$, i.e. $\min_\theta \mathbb{E}_{\mathbf{x},\mathbf{y}}\Big[d(f_\theta(\mathbf{y}),\mathbf{x})\Big]$,
where $d$ is an appropriate loss function. Most losses for denoising involve norms, such as the $\ell_2$ distance and $\ell_1$ norms, but these do not lead to estimates that provide MAP estimates. Instead, LPNs are trained with a \emph{proximal matching loss} with hyper-parameter $\gamma>0$, $d_\gamma(f_\theta(\mathbf{y}),\mathbf{x})$ \citep{fang2024whatspriorlearnedproximal}, given by
\begin{equation}
d_\gamma(f_\theta(\y),\x)
  = 1 - \frac{1}{(\pi \gamma^2)^{n/2}}
      \exp\!\left( -\frac{\|f_\theta(\y)-\x\|_2^2}{\gamma^2} \right).
\end{equation}
By minimizing the proximal matching loss, and in the limit as $\gamma\to 0$, the minimizer of this optimization problem recovers the proximal $\prox_{-\sigma^2\log{p_\x}}$ (see \citet[Theorem 3.2]{fang2024whatspriorlearnedproximal}).

\section{Learning Affine-Equivariant LPNs}
\label{ssec:AE-LPNs}

The methods presented above allow for training proximal operators particularly useful in inverse problems. However, it is unclear how to enforce additional structure in the learned priors, and equivariance in particular. In this section, we show how to construct learned proximal networks that respect affine equivariance, focusing on the specific case of affine equivariance, i.e. when $f(a\x + b \mathbf{1}) = af(\x) + b \mathbf{1}.$ Note that, in the context of inverse problems, it would not be sensible to seek for functions that respect complete affine-equivariance, i.e. $f(\mathbf{A} \x + \b) = \mathbf{A} f(\x) + \b$ for any $\mathbf{A} \in \mathbb{R}^{n \times n}$ and $\b \in \mathbb{R}^n$, as this leads to trivial operators\footnote{Indeed, when $\mathbf{A}=I$ and $\b = -\x$ equivariance is only satisfied for an identity up to a constant. Additionally, if $\mathbf{A}$ were a random permutation matrix, the denoiser would need to be equivariant to any permutation of pixels, which is not practically useful.}. We will study shift and scale equivariance separately, and then combine them into our AE-LPNs.

\paragraph{Scale-equivariant proximals.}
Designing scale-equivariant functions parametrized by neural networks--that are given by the composition of affine maps and entry-wise non-linearities--can be easily done by removing all bias terms and employing scale-equivariant (or 1-homogeneous) activation functions, such as Rectifying Linear Units (ReLU) \citep{mohan2020robustinterpretableblindimage} or SortPool \citep{herbreteau2024normalizationequivariantneuralnetworksapplication}. However, the gradient of a scale-equivariant function may not be scale-equivariant\footnote{Consider, for example, the identity function $f(x) = x$.}. Since our focus is on the design of proximal functions via LPNs, we require the gradient of a convex neural potential to be scale-equivariant, necessitating a different construction. We show in the following lemma that, if a function is homogeneous of degree $2$, then its gradient is scale-equivariant.

\begin{lemma}
Let the function $\Psi : \mathbb{R}^n \to \mathbb{R}$ be homogeneous of degree two, i.e., $\Psi(a\x) = a^2 \Psi(\x), \forall a>0$. Then, its gradient is homogeneous of degree 1 (i.e., scale-equivariant): $\nabla \Psi(a\x) = a \nabla \Psi(\x)$.
\end{lemma}

\begin{proof}
It suffices to compute the gradient of $\nabla_\x\Psi(a\x) = a\nabla \Psi(a\x)$, and to compare it to the gradient of $\nabla_\x [a^2\Psi(\x)] = a^2 \nabla \Psi(\x)$ to realize that, since $\Psi(a\x) = a^2 \Psi(\x)$ by assumption, then $\nabla \Psi(a\x) = a \nabla \Psi(\x)$.
\end{proof}

Importantly, this lemma is constructive: to obtain a scale-equivariant LPN, it suffices to make the ICNN homogeneous of degree 2. To this end, we first construct a scale-equivariant ICNN, $\Psi_\theta$, by removing all the bias terms and use scale-equivariant SortPool activations, as in \citet{herbreteau2024normalizationequivariantneuralnetworksapplication}. Then, we square the output of $\Psi_\theta$ to obtain $h_\theta(\x) = \Psi_\theta^2 (\x)$. It is evident that $h_\theta$
remains convex and is now homogeneous of degree two since
\[
    h_\theta(a\x) = \Psi_\theta^2(a\x) = (a \Psi_\theta(\x))^2 = a^2 h_\theta(\x).
\]
Thus, $\nabla h(\x)$ is a scale-equivariant proximal operator.

\paragraph{Affine-equivariant proximals.}
To achieve shift-equivariance, in addition to scale-equivariance, we construct a convex function that decomposes a signal from its mean by a projection onto the all-ones vector $\ones \in \mathbb{R}^n$ and its orthogonal complement $\ones^\perp$. More precisely, let $P = \tfrac{1}{n}\ones\ones^\top$ denote the projection onto $\operatorname{span}\{\mathbf{1}\}$,
so that $P\x$ results in a constant (mean) vector, while $(I-P)\x$ is its projection onto the orthogonal complement of $\operatorname{span}\{\mathbf{1}\}$. Then, given a convex, $2$-homogeneous function $h:\mathbb{R}^n \to \mathbb{R}$, we define
\begin{equation}\label{eq:AE-LPN}
    \tilde{h}(\x) = h((I-P)\x) + \tfrac{1}{2}\|P\x\|_2^2.
\end{equation}
This function is convex, as it is the sum of convex functions that are compositions of a convex function and an affine function.
Moreover, it is easy to see that $ \nabla \tilde h (\x) = (I-P)\nabla h((I-P)\x) + P\x$.
The two desired equivariances then follow. First, for an additive shift in the input by $c\mathbf{1}$, we have $\nabla \tilde h (\x+c\mathbf{1}) = \nabla \tilde h(\x)+c\mathbf{1}$.
Furthermore, the $2$-homogeneity of $h$ implies that $\tilde h$ is $2$-homogeneous also, ensuring that $\nabla \tilde h$ is scale-equivariant. In conclusion, $\nabla \tilde h(a\x+c\mathbf{1}) = a \nabla \tilde h(\x) + c\mathbf{1}.$
We have thus proved the following simple Lemma, which characterizes our AE-LPNs.

\begin{lemma}[AE-LPNs]
    Let $\tilde h(\x)$ be defined as in \eqref{eq:AE-LPN}. Then, there exist $R$ so that $\nabla \tilde h(\x) = \prox_R(\x)$ and, moreover, $\nabla \tilde h(\x)$ is affine-equivariant.
\end{lemma}

\paragraph{Practical construction of the AE-LPN.}
The architecture of the AE-LPN is similar to that of the original LPN developed by \citet{fang2024whatspriorlearnedproximal}. An Input Convex Neural Network (ICNN) is constructed by constraining its weights, or parameters, to be non-negative and employing non-decreasing activation functions, with SortPool chosen for scale equivariance \citep{NIPS2000_44968aec}. Different from the LPN in \citet{fang2024whatspriorlearnedproximal}, additive bias terms are removed from all layers and residual connections and the scalar output is squared, ensuring a degree-two homogeneous function. This ICNN defines the function $h$ applied to the input minus its mean, i.e. $(I-P)\x$. To
achieve full affine equivariance of the gradient, this function is combined with the quadratic term $\tfrac{1}{2}\|P\x\|_2^2$, where $P\x$ is a constant vector with each entry equal to the mean of the input. The resulting scalar function is precisely the one computed in \eqref{eq:AE-LPN}.
The AE-LPN proximal equivariant operator is then computed as the gradient of this function, such that $\hat{\x} = \nabla \tilde h_\theta(\x)$.

As with the original LPN framework, we train our AE-LPN with proximal matching loss, ensuring that $\nabla \tilde h_\theta(\x)$ approximates the MAP denoiser for signals $\x\sim p_\x$.
To this end, a clean sample $\mathbf{x} \sim p_\x$ is drawn from the training dataset and corrupted with Gaussian noise $\z \sim \mathcal{N}(0,\sigma^2I)$ at a noise-level $\sigma$ to yield $\mathbf{y} = \mathbf{x} + \z.$ The AE-LPN produces a denoised output via its gradient, as $\hat{\mathbf{x}} = \nabla \tilde h_\theta(\mathbf{y})$,
and a loss is computed to optimize the network's parameters. As in \citet{fang2024whatspriorlearnedproximal}, training begins with a brief pre-training phase using the $\ell_1$ loss. Afterward, optimization continues
using the proximal matching loss as presented in Section~\ref{sec:background}, where $\gamma > 0$ is gradually decreased with proceeding training epochs.

\section{Experimental Results}
\label{sec:results}

We study three settings: recovery of an affine-equivariant proximal, denoising across noise levels, and denoising under affine transformations, highlighting the equivariant properties in the AE-LPN.

\begin{figure}[t]
  \centering
  \includegraphics[width=0.95\textwidth]{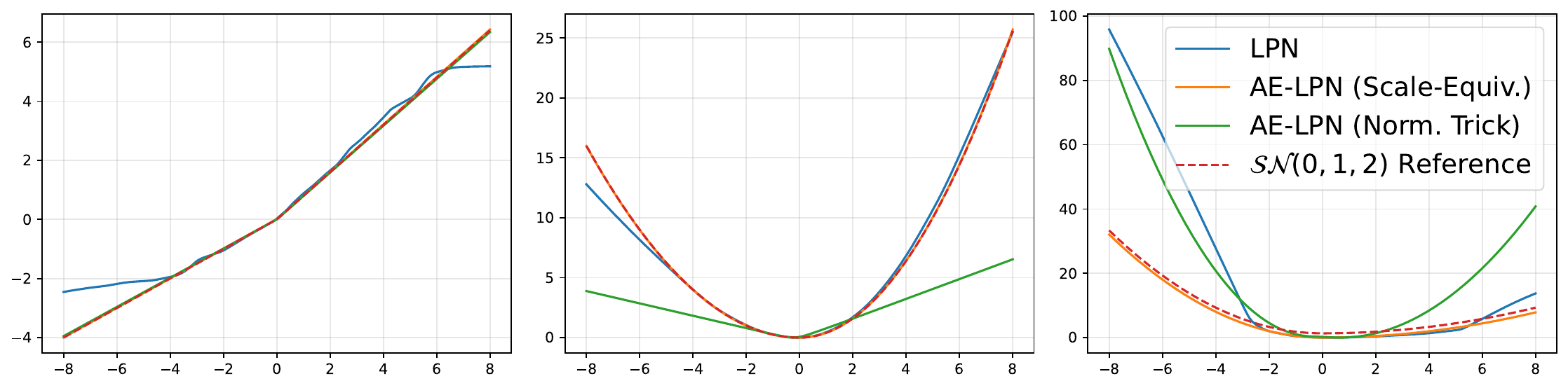}
  \caption{Learning a scale-equivariant proximal operator from samples via LPNs \citep{fang2024whatspriorlearnedproximal}. \textsc{Left}: learned operators, \textsc{middle}: convex potential functions, \textsc{right}: implicit regularizer.}
  \label{fig:split normal res}
\end{figure}

\paragraph{Learning proximal operators of a split normal.}
We first study a split normal distribution in one dimension, $\mathcal{SN}(\mu,\sigma_1,\sigma_2)$, defined as \citep{1973ApPhL..22..568G}
\[
p(x) = \dfrac{\sqrt{2}}{\sqrt{\pi}(\sigma_1+\sigma_2)}
\exp\!\left( -\dfrac{(x-\mu)^2}{2\sigma_i^2} \right),
\]
where $\sigma_i=\sigma_1$ when $x<\mu$, and $\sigma_i=\sigma_2$ when $x\geq\mu$.
The proximal operator of the negative log-prior is
\[
\operatorname{prox}_{-\lambda\log p}(x) =
\frac{\lambda \mu + \sigma_1^2 x}{\lambda + \sigma_1^2}\,\mathbf{1}_{\{x<\mu\}} +
\frac{\lambda \mu + \sigma_2^2 x}{\lambda + \sigma_2^2}\,\mathbf{1}_{\{x \ge \mu\}},
\]
which is scale-equivariant when $\mu=0$. Note that the only one-dimensional function that is affine-equivariant in 1D is $f(x)=x$, which may be trivial to learn as a proximal operator. We then train an original LPN, an LPN with the normalization-trick described in Section~\ref{sec:intro}, our proposed AE-LPN, and, as an ablation, LPNs that are equivariant only to scale or shift, drawing i.i.d samples from $\mathcal{SN}(0,1,2)$ and additive Gaussian noise with $\sigma=1.0$. In proximal matching we use $\gamma=0.1$. The weights are optimized with Adam \citep{kingma2017adammethodstochasticoptimization} for 10k steps at $lr=10^{-3}$ and 10k with $lr=10^{-4}$.

As seen in Figure~\ref{fig:split normal res}, only the AE-LPN can learn the true proximal operator and log-prior, providing a prox that is scale-equivariant. While the normalization trick provides a scale-equivariant denoiser, its mapping is no longer ensured to be an exact proximal operator (as it is not the gradient of a convex potential). Such a guarantee is essential in many settings, e.g., convergence of Plug-n-Play iterations in solving inverse problems \citep{6737048}.
In contrast, AE-LPNs achieve equivariance while retaining the guarantee that the learned mapping is an exact proximal operator.

\begin{figure}[t]
  \centering
  \includegraphics[trim=0 55 0 0, width=0.75\textwidth]{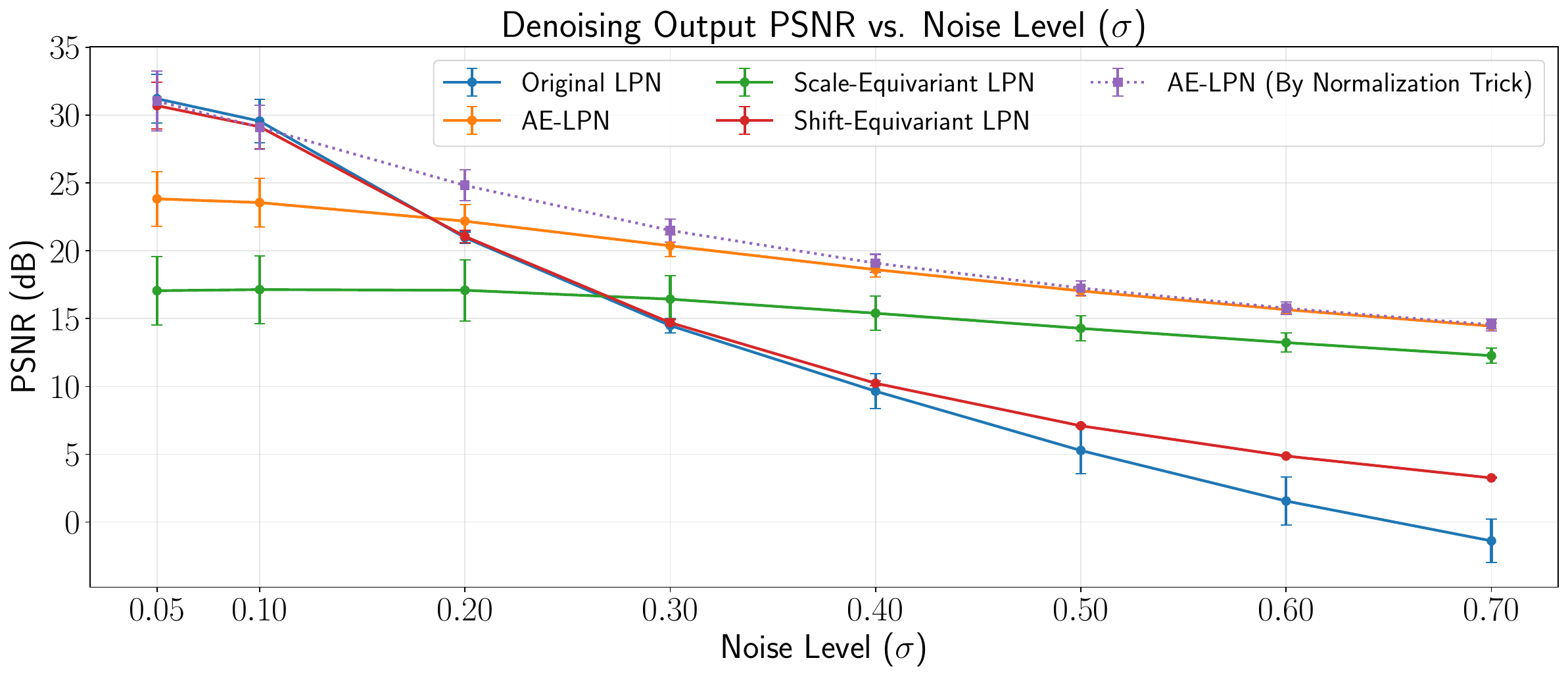}
  \caption{PSNR of different methods (see description in text) for denoising across varying noise levels on the BSDS500 dataset. All models are trained with noise level of 0.1.}
  \label{fig:denoising test plot}
\end{figure}

\paragraph{Image denoising across noise-levels.}
In image denoising, affine-equivariance can help alleviate sensitivity of denoisers to distribution shifts \citep{herbreteau2024normalizationequivariantneuralnetworksapplication}. We use the BSDS500 dataset \citep{amfm_pami2011}, containing a variety of general natural images. For consistency, we cropped the images to a resolution of 128$\times$128 and implemented architectures at this size. All models were trained using Gaussian noise level of $0.1$, pre-training using $\ell_1$ loss for 20k steps ($lr = 10^{-5}$ for equivariant models and $lr=10^{-3}$ for LPN) followed by 20k steps with proximal matching ($lr = 10^{-5}$ for equivariant models and $lr=10^{-4}$ for LPN) with $\gamma$ starting at $\gamma=0.64\sqrt{128^2\times3} \approx 142$ and halved every 5k steps.

\cref{fig:denoising test plot} shows denoising PSNR across input noise levels for five models. Naturally, AE-LPNs' performance is lower than that of more flexible LPNs when tested at the same noise level seen during training (0.1). However, equivariant models are significantly more robust to changes in the noise level at testing time. In this denoising setting, scale equivariance seems to be more important than shift equivariance, as evidenced by the green and red curves. Lastly, while we include the equivariant model via the normalization trick, such a model does not provide a proximal operator.

\begin{figure}[t]
  \centering
  \includegraphics[trim=0 80 0 0, width=0.75\textwidth]{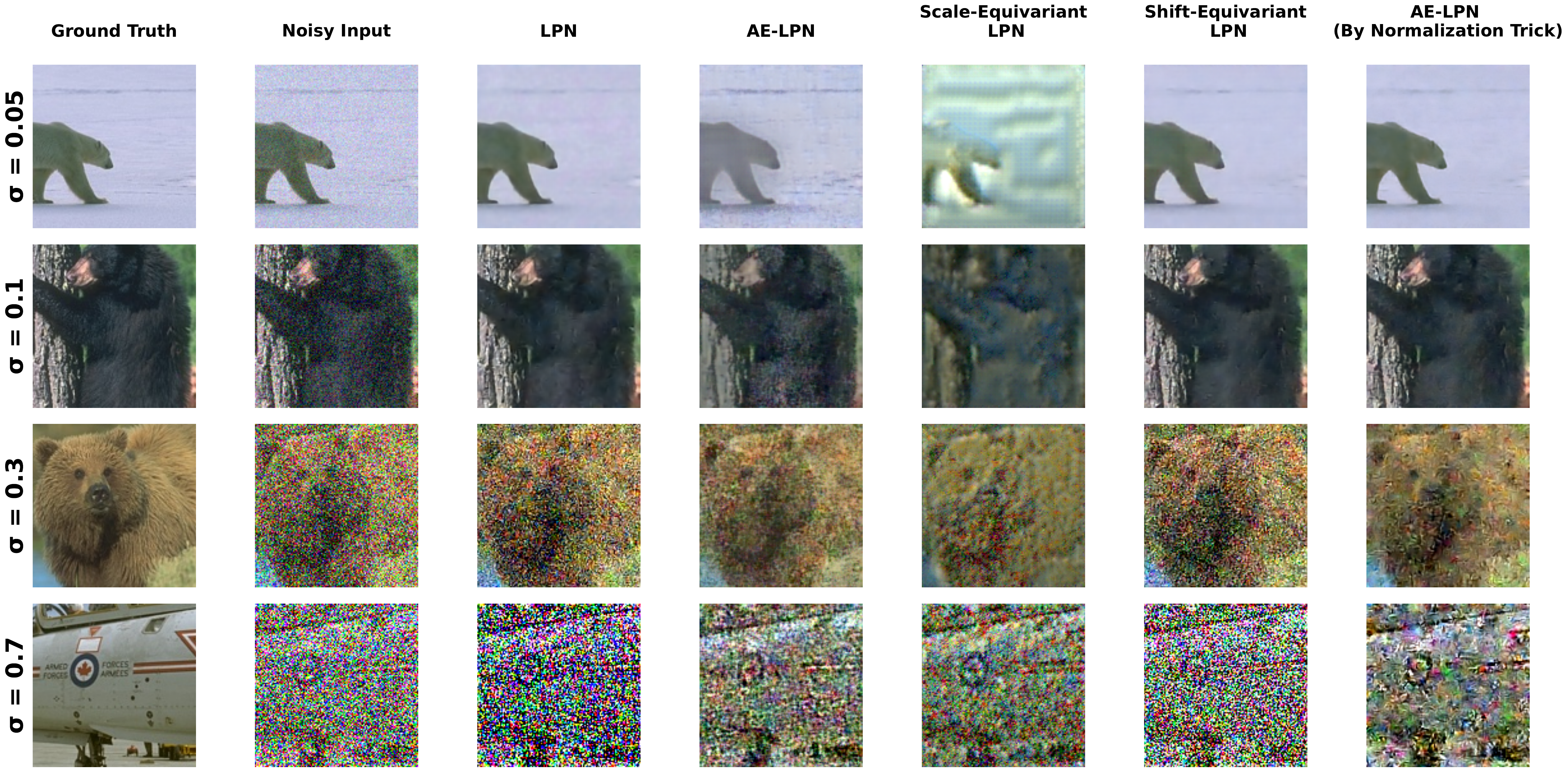}
  \caption{Example results on BSDS500 across noise levels.}
  \label{fig:denoising visual chart}
\end{figure}

\paragraph{Evaluating affine transformations.}
Finally, we show that AE-LPNs are the only model that fully maintains the affine structure of an image while learning a true proximal operator. We apply affine transformations $g$ on images and compute the PSNR between $f_\theta(g(\x))$ and $g(f_\theta(\x))$. Figure~\ref{fig:affine_denoising_curves} shows these results for the transformation $g(\x) = \alpha \x + (1-\alpha)$ for values of $\alpha \in (0,1]$, simulating changes in image brightness. Results show that AE-LPNs can both learn a true proximal operator and account for the scale and shift of the input, as evidenced by the very high PSNR (i.e., negligible error) between $f_\theta(g(\x))$ and $g(f_\theta(\x))$.

\begin{figure}[t]
  \centering
  \includegraphics[trim=0 30 0 0, width=0.5\textwidth]{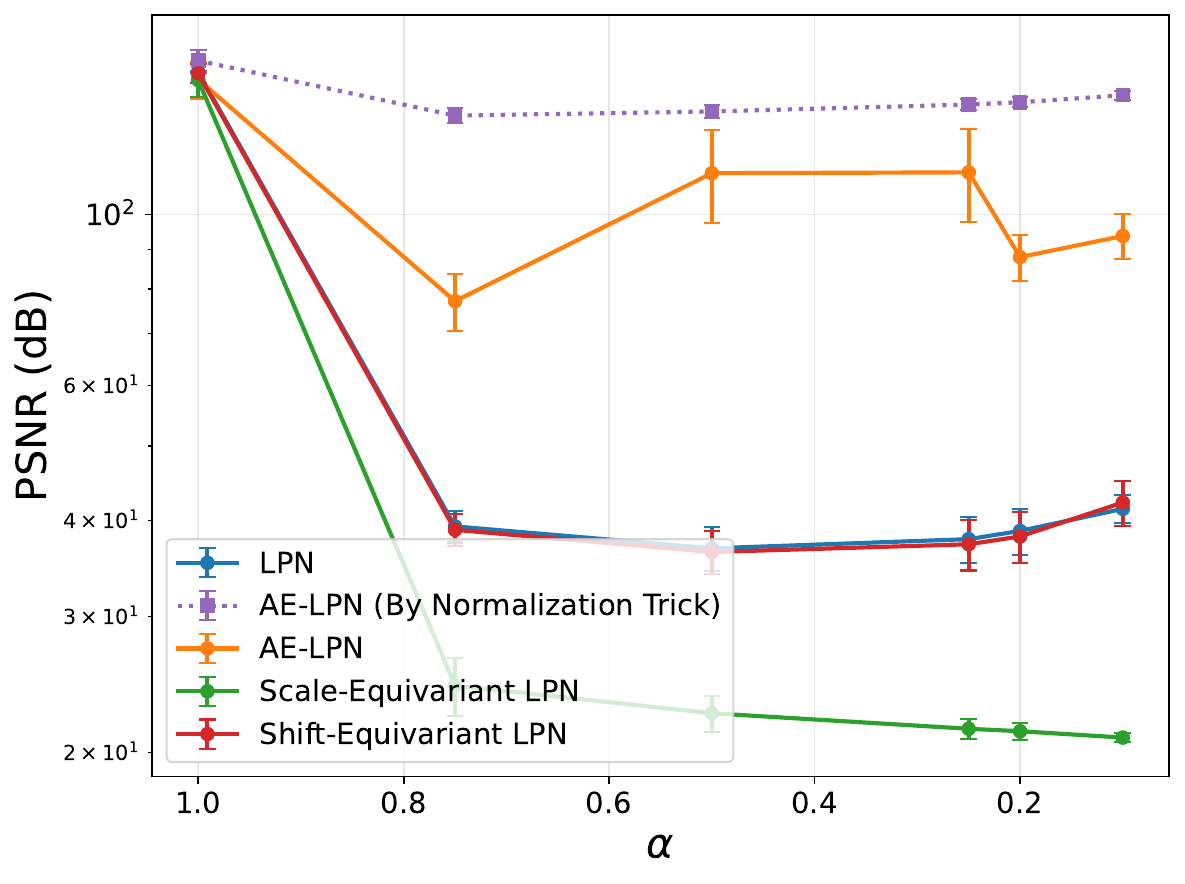}
  \caption{Results validating equivariance on BSDS500 under the transformation $g(\x) = \alpha\x + (1-\alpha)$.}
  \label{fig:affine_denoising_curves}
\end{figure}

\section{Conclusion}
\label{sec:conclusion}

This work presented a framework for constructing exact proximal operators that are equivariant to affine transformations, dubbed Affine-Equivariant Learned Proximal Networks (AE-LPNs). These models provide data-driven equivariant proximal operators that are beneficial when the underlying data exhibits equivariant structures, and improve robustness to noise levels beyond the training distribution for image denoising. The benefits of AE-LPNs could extend beyond denoising to improved sample complexity for training and applications like CT reconstruction, where variations in windowing and intensity scaling are ubiquitous, which remain future work.

\section*{Acknowledgments}
JS and ZF acknowledge funding from NIH Grant P41EB031771. JS is partially supported by NSF DMS 2502377.

\bibliographystyle{plainnat}
\bibliography{strings}

@misc{fang2024whatspriorlearnedproximal,
      title={What's in a Prior? Learned Proximal Networks for Inverse Problems}, 
      author={Zhenghan Fang and Sam Buchanan and Jeremias Sulam},
      year={2024},
      eprint={2310.14344},
      archivePrefix={arXiv},
      primaryClass={cs.CV},
      url={https://arxiv.org/abs/2310.14344}, 
}

@misc{amos2017inputconvexneuralnetworks,
      title={Input Convex Neural Networks}, 
      author={Brandon Amos and Lei Xu and J. Zico Kolter},
      year={2017},
      eprint={1609.07152},
      archivePrefix={arXiv},
      primaryClass={cs.LG},
      url={https://arxiv.org/abs/1609.07152}, 
}

@misc{gribonval2020characterizationproximityoperators,
      title={A characterization of proximity operators}, 
      author={Rémi Gribonval and Mila Nikolova},
      year={2020},
      eprint={1807.04014},
      archivePrefix={arXiv},
      primaryClass={math.CA},
      url={https://arxiv.org/abs/1807.04014}, 
}

@article{bietti2019group,
  title={Group invariance, stability to deformations, and complexity of deep convolutional representations},
  author={Bietti, Alberto and Mairal, Julien},
  journal={Journal of Machine Learning Research},
  volume={20},
  number={25},
  pages={1--49},
  year={2019}
}

@article{anselmi2016unsupervised,
  title={Unsupervised learning of invariant representations},
  author={Anselmi, Fabio and Leibo, Joel Z and Rosasco, Lorenzo and Mutch, Jim and Tacchetti, Andrea and Poggio, Tomaso},
  journal={Theoretical Computer Science},
  volume={633},
  pages={112--121},
  year={2016},
  publisher={Elsevier}
}

@inproceedings{Meinhardt_2017,
   title={Learning Proximal Operators: Using Denoising Networks for Regularizing Inverse Imaging Problems},
   url={http://dx.doi.org/10.1109/ICCV.2017.198},
   DOI={10.1109/iccv.2017.198},
   booktitle={2017 IEEE International Conference on Computer Vision (ICCV)},
   publisher={IEEE},
   author={Meinhardt, Tim and Moeller, Michael and Hazirbas, Caner and Cremers, Daniel},
   year={2017},
   month=oct, pages={1799–1808} }

@misc{mohan2020robustinterpretableblindimage,
      title={Robust and interpretable blind image denoising via bias-free convolutional neural networks}, 
      author={Sreyas Mohan and Zahra Kadkhodaie and Eero P. Simoncelli and Carlos Fernandez-Granda},
      year={2020},
      eprint={1906.05478},
      archivePrefix={arXiv},
      primaryClass={eess.IV},
      url={https://arxiv.org/abs/1906.05478}, 
}

@inproceedings{NIPS2000_44968aec,
 author = {Dugas, Charles and Bengio, Yoshua and B\'{e}lisle, Fran\c{c}ois and Nadeau, Claude and Garcia, Ren\'{e}},
 booktitle = {Advances in Neural Information Processing Systems},
 editor = {T. Leen and T. Dietterich and V. Tresp},
 pages = {},
 publisher = {MIT Press},
 title = {Incorporating Second-Order Functional Knowledge for Better Option Pricing},
 url = {https://proceedings.neurips.cc/paper_files/paper/2000/file/44968aece94f667e4095002d140b5896-Paper.pdf},
 volume = {13},
 year = {2000}
}

@ARTICLE{1973ApPhL..22..568G,
       author = {{Gibbons}, J.~F. and {Mylroie}, S.},
        title = "{Estimation of impurity profiles in ion-implanted amorphous targets using joined half-Gaussian distributions}",
      journal = {Applied Physics Letters},
         year = 1973,
        month = jun,
       volume = {22},
       number = {11},
        pages = {568-569},
          doi = {10.1063/1.1654511},
       adsurl = {https://ui.adsabs.harvard.edu/abs/1973ApPhL..22..568G}
}

@misc{kingma2017adammethodstochasticoptimization,
      title={Adam: A Method for Stochastic Optimization}, 
      author={Diederik P. Kingma and Jimmy Ba},
      year={2017},
      eprint={1412.6980},
      archivePrefix={arXiv},
      primaryClass={cs.LG},
      url={https://arxiv.org/abs/1412.6980}, 
}

@misc{evangelista2023ambiguitysolvingimaginginverse,
      title={Ambiguity in solving imaging inverse problems with deep learning based operators}, 
      author={Davide Evangelista and Elena Morotti and Elena Loli Piccolomini and James Nagy},
      year={2023},
      eprint={2305.19774},
      archivePrefix={arXiv},
      primaryClass={cs.CV},
      url={https://arxiv.org/abs/2305.19774}, 
}

@misc{milanfar2024denoisingpowerfulbuildingblockimaging,
      title={Denoising: A Powerful Building-Block for Imaging, Inverse Problems, and Machine Learning}, 
      author={Peyman Milanfar and Mauricio Delbracio},
      year={2024},
      eprint={2409.06219},
      archivePrefix={arXiv},
      primaryClass={cs.LG},
      url={https://arxiv.org/abs/2409.06219}, 
}

@misc{candes2012mathematicaltheorysuperresolution,
      title={Towards a Mathematical Theory of Super-Resolution}, 
      author={Emmanuel Candes and Carlos Fernandez-Granda},
      year={2012},
      eprint={1203.5871},
      archivePrefix={arXiv},
      primaryClass={cs.IT},
      url={https://arxiv.org/abs/1203.5871}, 
}

@ARTICLE{8015117,
  author={Oliveri, Giacomo and Salucci, Marco and Anselmi, Nicola and Massa, Andrea},
  journal={IEEE Antennas and Propagation Magazine}, 
  title={Compressive Sensing as Applied to Inverse Problems for Imaging: Theory, Applications, Current Trends, and Open Challenges}, 
  year={2017},
  volume={59},
  number={5},
  pages={34-46},
  doi={10.1109/MAP.2017.2731204}}

@misc{benning2018modernregularizationmethodsinverse,
      title={Modern Regularization Methods for Inverse Problems}, 
      author={Martin Benning and Martin Burger},
      year={2018},
      eprint={1801.09922},
      archivePrefix={arXiv},
      primaryClass={math.NA},
      url={https://arxiv.org/abs/1801.09922}, 
}

@article{doi:10.1137/1021044,
author = {Willoughby, Ralph A.},
title = {Solutions of Ill-Posed Problems (A. N. Tikhonov and V. Y. Arsenin)},
journal = {SIAM Review},
volume = {21},
number = {2},
pages = {266-267},
year = {1979},
doi = {10.1137/1021044},
URL = { 
        https://doi.org/10.1137/1021044
},
eprint = {   
        https://doi.org/10.1137/1021044    
}
}

@inproceedings{NIPS2008_c16a5320,
 author = {Jain, Viren and Seung, Sebastian},
 booktitle = {Advances in Neural Information Processing Systems},
 editor = {D. Koller and D. Schuurmans and Y. Bengio and L. Bottou},
 pages = {},
 publisher = {Curran Associates, Inc.},
 title = {Natural Image Denoising with Convolutional Networks},
 url = {https://proceedings.neurips.cc/paper_files/paper/2008/file/c16a5320fa475530d9583c34fd356ef5-Paper.pdf},
 volume = {21},
 year = {2008}
}

@article{10.1561/2400000003,
author = {Parikh, Neal and Boyd, Stephen},
title = {Proximal Algorithms},
year = {2014},
issue_date = {Jan 2014},
publisher = {Now Publishers Inc.},
address = {Hanover, MA, USA},
volume = {1},
number = {3},
issn = {2167-3888},
url = {https://doi.org/10.1561/2400000003},
doi = {10.1561/2400000003},
journal = {Found. Trends Optim.},
month = jan,
pages = {127–239},
numpages = {116}
}

@INPROCEEDINGS{6737048,
  author={Venkatakrishnan, Singanallur V. and Bouman, Charles A. and Wohlberg, Brendt},
  booktitle={2013 IEEE Global Conference on Signal and Information Processing}, 
  title={Plug-and-Play priors for model based reconstruction}, 
  year={2013},
  volume={},
  number={},
  pages={945-948},
  doi={10.1109/GlobalSIP.2013.6737048}}

@misc{romano2017littleenginecouldregularization,
      title={The Little Engine that Could: Regularization by Denoising (RED)}, 
      author={Yaniv Romano and Michael Elad and Peyman Milanfar},
      year={2017},
      eprint={1611.02862},
      archivePrefix={arXiv},
      primaryClass={cs.CV},
      url={https://arxiv.org/abs/1611.02862}, 
}

@misc{herbreteau2024normalizationequivariantneuralnetworksapplication,
      title={Normalization-Equivariant Neural Networks with Application to Image Denoising}, 
      author={Sébastien Herbreteau and Emmanuel Moebel and Charles Kervrann},
      year={2024},
      eprint={2306.05037},
      archivePrefix={arXiv},
      primaryClass={cs.CV},
      url={https://arxiv.org/abs/2306.05037}, 
}

@Article{amfm_pami2011,
 author = {Arbelaez, Pablo and Maire, Michael and Fowlkes, Charless and Malik, Jitendra},
 title = {Contour Detection and Hierarchical Image Segmentation},
 journal = {IEEE Trans. Pattern Anal. Mach. Intell.},
 issue_date = {May 2011},
 volume = {33},
 number = {5},
 month = may,
 year = {2011},
 issn = {0162-8828},
 pages = {898--916},
 numpages = {19},
 url = {http://dx.doi.org/10.1109/TPAMI.2010.161},
 doi = {10.1109/TPAMI.2010.161},
 acmid = {1963088},
 publisher = {IEEE Computer Society},
 address = {Washington, DC, USA},
}

@article{bruckstein2009sparse,
  title={From sparse solutions of systems of equations to sparse modeling of signals and images},
  author={Bruckstein, Alfred M and Donoho, David L and Elad, Michael},
  journal={SIAM review},
  volume={51},
  number={1},
  pages={34--81},
  year={2009},
  publisher={SIAM}
}

\newpage
\end{document}